\documentclass[runningheads]{llncs}
\usepackage[T1]{fontenc}
\usepackage{graphicx}
\usepackage{amsmath}
\usepackage{amsfonts}
\usepackage{enumitem}
\usepackage[mode=multiuser, layout=inline, status=draft]{fixme}
\usepackage[linesnumbered,ruled,vlined]{algorithm2e}
\usepackage{float}


\fxsetup{theme=color}
\FXRegisterAuthor{ml}{aml}{\color{blue}Björn}
\FXRegisterAuthor{oo}{aoo}{\color{blue}Özgür}
\FXRegisterAuthor{r}{ar}{\color{red}Raus oder kürzen: }
\DeclareMathOperator*{\argmax}{arg\,max} 
\begin{document}
%
%\ title{Data Evaluation and Incentivizing of Data Sharing for Collaborative Causal Inference}
% \title{Data Sharing for Collaborative Causal Inference: Dealing with Fairness and Incentives}
\title{Mechanisms for Data Sharing in Collaborative Causal Inference (Extended Version)}
\titlerunning{Mechanisms for Data Sharing in Collaborative Causal Inference}
\author{Björn Filter\inst{1}\orcidID{0009-0008-8666-6239} \and
Ralf Möller \inst{2}\orcidID{0000-0002-1174-3323} \and
Özgür Lütfü Özçep \inst{2}\orcidID{0000-0001-7140-2574} 
}
\authorrunning{B. Filter et al.}
\institute{Institute of Information Systems, University of Lübeck, Germany
\email{b.filter@uni-luebeck.de}\and
Institute for Humanities-Centered AI (CHAI), University of Hamburg, Germany
\email{\{ralf.moeller,oezguer.oezcep\}@uni-hamburg.de}
}
\maketitle              % typeset the header of the contribution
\begin{abstract}
Collaborative causal inference (CCI) is a federated learning method for pooling data from multiple, often self-interested, parties, to achieve a common learning goal over causal structures, e.g. estimation and optimization of treatment variables in a medical setting. Since obtaining data can be costly for the participants and sharing unique data poses the risk of losing competitive advantages, motivating the participation of all parties through equitable rewards and incentives is necessary. This paper devises an evaluation scheme to measure the value of each party's data contribution to the common learning task, tailored to causal inference's statistical demands, by comparing completed partially directed acyclic graphs (CPDAGs) inferred from observational data contributed by the participants. The Data Valuation Scheme thus obtained can then be used to introduce mechanisms that incentivize the agents to contribute data. It can be leveraged to reward agents fairly, according to the quality of their data, or to maximize all agents' data contributions. 

\end{abstract}

\section{Introduction}
Causal inference estimates the causal effect of treatment variables within a particular population, a method widely adopted across various fields. In healthcare, for example, it assesses the effectiveness of pharmaceutical interventions \cite{shi2022learning} and explores gene effects on phenotypes \cite{hu2018application}. It also finds application in policy decision-making \cite{yao2021survey}, recommender systems \cite{liang2016causal}, education \cite{cordero2018causal}, advertisement \cite{sun2015causal}, agriculture \cite{singer2014human}, and numerous other domains.

Various methodologies for causal inference have emerged to analyze interventional data, obtained from controlled experimental trials, or observational data. However, their effectiveness can be hindered by data quality issues, including data sparsity and non-representativeness. For example, patient preferences for hospitals can limit each institution's records, causing biases and data sparsity. Thus, causal inference efforts by single hospitals risk inaccurate treatment effect estimates, potentially impeding optimal medication prescriptions \cite{masic2008evidence}.

Collaborative causal inference (CCI) leverages collective data pooling from diverse agents, such as companies, organizations, or individuals, to tackle data sparsity and non-representativeness challenges. This approach enhances the precision of causal model estimations. Methods such as simple aggregation or multi-source causal inference \cite{bareinboim2016causal} yield accurate and statistically significant estimates by utilizing data from all participants. However, success depends on the willingness of agents to share data, which is not always guaranteed. Entities often prioritize self-interest, hesitating to share data due to associated collection costs.

As a concrete example consider the case of genetic research for predicting and treating illnesses which relies on substantial amounts of data. 
Despite decreasing costs, acquiring such data remains expensive \cite{pareek2011sequencing}. Pooling genome and gene expression data from multiple sources is crucial for accurate models, especially when dealing with biased subpopulations. To achieve this, assessing data quality and incentivizing parties based on it are necessary.

In this paper, we propose a mechanism for CCI aimed at motivating agents to provide maximal amounts of data to gain access to a high-quality model. We introduce a quantitative measure to assess data quality and each party's contribution to the collective model. This valuation scheme is tailored to the unique requirements of statistical causal inference, where, in contrast to the framework proposed by Karimireddy and colleagues \cite{karimireddy2022mechanisms}, no ground truth is available. Unlike the data valuation scheme introduced by Qiao and colleagues \cite{pmlr-v202-qiao23a}, our method considers differences in estimated causal graph models and probability distributions for causal effects among contributors' data. The resulting data valuation function incentivizes agents to maximize their data contribution. Furthermore, it can also be used for the allocation of fair incentives based on data quality, ensuring better causal estimators for agents with higher-quality data.

An implementation of our data valuation scheme can be found at\\ https://github.com/bjofilter/Data-Valuation-CCI. 
This paper is an extended version of a paper to be published in the proceedings of the 47th German AI conference (KI 2024).    

\section{Preliminaries}\label{prelim}
We collect the basic notions and methods for causal structure learning. 
%using directed graphical models. 
%A graph $\mathcal{G} = (\mathbf{V}, \mathbf{E})$ consists of a set of \textit{vertices} $\mathbf{V} = \{1, ..., p\}$ and a set of \textit{edges} $\mathbf{E} \subseteq \mathbf{V} \times \mathbf{V}$. An edge between two nodes $V_i, V_j \in \mathbf{V}$ is called \textit{undirected}, if both $(V_i, V_j), (V_j, V_i) \in \mathbf{E}$. An edge is called \textit{directed} from $V_i$ to $V_j$,  if $(V_i, V_j) \in \mathbf{E}$, but $(V_j, V_i) \notin \mathbf{E}$. In this case, $V_i$ is called the parent of $V_j$ and $V_j$ the child of $V_i$. Let $Pa_{V_i}(\mathcal{G})$ denote the set of parents of node $V_i$ in graph $\mathcal{G}$. A graph is called \textit{directed} if it contains directed edges only. A directed graph, which contains no directed cycle, is called a \textit{directed acyclic graph (DAG)}.
\textit{Directed acyclic graphs (DAGs)} can model 
%be used as mathematical models 
%to visualize 
conditional independencies. For a DAG $\mathcal{D} = (\mathbf{V, E})$ let the vertices $\mathbf{V} = {V_1, ..., V_p}$ correspond to a $p$-variate random vector $\mathbf{X} \in \mathbb{R}^p$, while the edges $\mathbf{E}$ denote relationships between the variables they connect. A probability function $P$ over $\mathbf{X}$ is called \textit{Markov} relative to a DAG $\mathcal{D}$ if it factorizes as 
$P(x_1, ..., x_p) = \prod_{i = 1}^{p} P(x_i \mid Pa_{X_i}(\mathcal{D}))$, where $Pa_{X_i}(\mathcal{D})$ are the parents of  node $X_i$. 

In causal inference, a DAG $\mathcal{D}$ is not only used to represent associational but also causal relationships. Here, intuitively, the parents $Pa_{V_i}(\mathcal{D})$ of a node $V_i \in \mathbf{V}$ can be interpreted as the \textit{direct causes} of $V_i$, and the children as the \textit{direct effects} of $V_i$. Such a model gives information about how a variable would change, if its direct causes are changed, for example by external manipulation. To denote such a manipulation, Pearl's \textit{do-operator} can be used: Let $\mathcal{D} = (\mathbf{V}, \mathbf{E})$ be a DAG that is Markov to a probability distribution $P$ over a set of random variables $\mathbf{V} = \lbrace V_1, ..., V_p \rbrace$. Let $V_i, V_j \in \mathbf{V}$ be two distinct variables and $v_i,  v_i$ the values of those variables. The \textit{causal effect} of $V_i$ on $V_j$ is denoted as $P(v_j \mid  do(V_i = v_i))$, where $do(V_i = v_i)$ or short $do(v_i)$ is an intervention that sets $V_i = v_i$ \cite{Pearl09}.

Let $\mathbf{v} = (v_1, ..., v_p)$ be a sequence of values of the variables $\mathbf{V}$. $\mathbf{v}$ is consistent with $v'_i$ if $v_i = v_i'$. Given a DAG $\mathcal{D}$, the distribution after an intervention, called \textit{intervention distribution}, can be expressed as follows: $P_{\mathcal{D}}(\mathbf{v} \mid do(V_i = v'_i)) =$ $\prod_{j: V_j \neq V_i} P(v_j \mid Pa_{V_j}(\mathcal{D}))$ if $\mathbf{v}$ is consistent with $v'_i$, else set $=0$.

The \textit{total causal effect} of $V_i$ on $V_j$ is written as $P(V_j \mid do(V_i))$ and is defined as the probability distribution of $V_j$ after the intervention at $V_i$.

A \textit{linear structural equation model (SEM)} $\mathcal{M}$ comprises endogenous variables $\mathbf{V} = \lbrace V_1, ..., V_p \rbrace$,  exogenous variables $\mathbf{U} = \lbrace U_1, ..., U_p \rbrace$, and a joint distribution $P_M$ over $\mathbf{U}$, with structural assignments $V_j = \sum_{i < j} c_{ij}V_i + U_j$ (for  $j = 1,...,p$) defining direct causal effects between variables. When dealing with observational data that is assumed to have been created by an underlying SEM, various algorithms exist to estimate the underlying DAG, like for example the PC algorithm \cite{kalisch2007estimating} or GES (greedy equivalence search) \cite{chickering2002optimal}. However, most algorithms for estimating a DAG from observational data cannot distinguish between two DAGs that are Markov equivalent (two DAGs $\mathcal{G}_1$ and $\mathcal{G}_2$ are said to be Markov equivalent if they encode the same set of conditional independencies \cite{chickering2013}). Instead, one can only identify the Markov equivalence class (MEC) of DAGs. This is represented by a \emph{Completed Partially Directed Acyclic Graph (CPDAG)}, which has a directed edge between two nodes if such an edge is present in all the DAGs belonging to the MEC and an undirected edge between two nodes if edges in both directions are present among DAGs in the MEC. The set of \textit{consistent DAG extensions} of $\mathcal{C}$, $CE(\mathcal{C})$, consists of all DAGs represented by $\mathcal{C}$. We define the \textit{intervention distribution of causal effects} given a CPDAG $\mathcal{C}$ as follows:
\begin{align}
    P_{\mathcal{C}}(\mathbf{v} \mid do(V_i = v'_i)) = \sum_{\mathcal{D}\in CE(\mathcal{C})}\frac{1}{|CE(\mathcal{C})|} P_{\mathcal{D}}(\mathbf{v} \mid do(V_i = v'_i))
\end{align}

When a CPDAG $\mathcal{C}$ has been estimated, the IDA algorithm (Intervention-calculus when the DAG is Absent) \cite{maathuis2009estimating} can be used to estimate the causal effect of a variable $V_i \in \mathbf{V}$ on another variable $V_j \in \mathbf{V}$. This algorithm estimates the multiset $\Theta^{ij}$ of all possible causal effects of $V_i$ on $V_j$ given the MEC described by $\mathcal{C}$, i.e. the effects described by all possible DAGs $\mathcal{D} \in CE(\mathcal{C})$. The basic idea behind the IDA algorithm is to determine all DAGs $\mathcal{D} \in CE(\mathcal{C})$ for a given CPDAG $\mathcal{C}$. In each of these DAGs, the parents of $V_i$ can be used as a so-called \textit{adjustment set} \cite{Pearl09} (see also appendix \ref{appdSID}) to estimate the \emph{average causal effect (ACE)} $\tau^{ij}(\mathcal{D})$, which is defined as 
%\begin{align}
 $   \tau^{ij}(\mathcal{D}) = E(V_j \mid do(V_i = v_i + 1)) - E(V_j \mid do(V_i = v_i))$,
%\end{align}
% \begin{align}
%     \tau^{ij}(\mathcal{D}) = E(V_j \mid do(V_i = v_i + 1)) - E(V_j \mid do(V_i = v_i)),
% \end{align}
where $ E(V_j \mid do(V_i = v_i))$ stands for the expected value of $V_j$ after the intervention $do(v_i)$. The ACE can be estimated by performing the regression of $V_j$ on $V_i$ and $Pa_{V_i}(\mathcal{D})$, denoted $\beta_{V_i|Pa_{V_i}(\mathcal{D})}$ as follows:
\begin{align}
    \beta_{V_i|Pa_{V_i}(\mathcal{D})} = \begin{cases} 0 &\text{if $V_j \in Pa_{V_i}(\mathcal{D})$}, \\
\text{coefficient of $V_i$ in $V_j \sim V_i + Pa_{V_i}(\mathcal{D})$} &\text{if $V_j \notin Pa_{V_i}(\mathcal{D})$},
\end{cases}
\end{align}
where $V_j \sim V_i + Pa_{V_i}(\mathcal{D})$ stands for the linear regression of $V_j$ on $V_i$ and $Pa_{V_i}(\mathcal{D})$ \cite{maathuis2009estimating}. Thus $\beta_{V_i|Pa_{V_i}(\mathcal{D})}$ describes a Gaussian distribution $\hat{\theta}^{ij}_{Pa_{V_i}(\mathcal{D})}$ consisting of the estimated effect $\hat{\tau}^{ij}_{Pa_{V_i}(\mathcal{D})}$ and the standard error $\hat{\sigma}^{ij}_{Pa_{V_i}(\mathcal{D})}$. Since only the parents of $V_i$ are used to estimate the effect, DAGs in which $V_i$ has the same parents yield the same effect. However, in DAGs with different parents, the estimated causal effect can be different. Thus, a regression has to be performed for each of the possible parent sets of $V_i$ given $\mathcal{C}$. Let $\mathbf{Pa}_{V_i}(\mathcal{C}) = \{Pa_{V_i}(\mathcal{D})|\mathcal{D} \in CE(\mathcal{C})\}$. The estimated causal effect of $V_i$ on $V_j$ given CPDAG $\mathcal{C}$ is given by the weighted sum over all possible sets of parents for $V_i$:
\begin{align}
    \hat{\theta}^{ij} = \sum_{\mathbf{Z} \in \mathbf{Pa}_{V_i}(\mathcal{C})} \frac{|\{\mathcal{D} \in CE(\mathcal{C})|Pa_{V_i}(\mathcal{D}) = \mathbf{Z}\}|}{|CE(\mathcal{C})|} \hat{\theta}^{ij}_{\mathbf{Z}},\label{mixmodel}
\end{align}
that is, the mixture model obtained by summing over all effects induced by a parent-set in $\mathbf{Pa}_{V_i}(\mathcal{C})$, divided by the fraction of DAGs in $CE(\mathcal{C})$ in which $V_i$ has that set of parents.

\section{Collaborative Causal Inference}\label{CCI}
We assume, that $n$ self-interested agents $N:=\{1, ..., n\}$ want to solve a common causal inference problem by estimating a CPDAG and using observational data of $p$ variables, $\mathbf{V} = \{V_1, ..., V_p\}$ to estimate causal effects between variable-pairs $(V_i, V_j) \in \mathbf{V} \times \mathbf{V}$. Furthermore, agents are non-malicious and they may acquire data from different and potentially biased populations, but these data collectively form the common target population of interest. Each agent $i \in N$ can produce data points $D_i$ for a fixed cost $c_i > 0$ per data point. Both single agents as well as coalitions of agents use their available data to infer a CPDAG $\mathcal{C} = \left(\mathbf{V}, \mathbf{E}_{\mathcal{C}} \right)$  and then infer the causal effects using the IDA algorithm. We will make the rather strict assumption, that all agents are equally interested in all effects between all possible variable pairs $(V_i, V_j) \in \mathbf{V} \times \mathbf{V}$.

Let $N$ denote the grand coalition containing all agents and let $C \subseteq N$ denote some coalition of agents. 
 $D_C$ denotes the data provided by coalition $C$. 
Let $\mathcal{E}(D_C) = \mathcal{E}_C = \left(\mathcal{C}_C, \hat{\Theta}_C\right)$ be the estimator computed from $D_C$ as described in Section \ref{prelim}. It consists of an estimated CPDAG $\mathcal{C}_C$ and the estimated distributions of causal effects $\hat{\Theta}_C = \{\hat{\theta}^{ij}_C \mid (V_i, V_j) \in \mathbf{V} \times \mathbf{V}\}$. To simplify notation, the subscript $C$ is replaced by index $i$ when $C = \{i\}$. Thus, $D_C = \cup_{i \in C}D_i$. 

In the following, it will be necessary for agents and coalitions, to know the quality of their estimators, to calculate their improvement rate, and to allocate fair incentives to agents. Therefore let $v:\mathcal{E}_{\mathcal{C}} \times \mathcal{E}_B \rightarrow \mathbb{R}$ be the valuation function for the estimator $\mathcal{E}_{\mathcal{C}}$, which will be defined in the next chapter. It evaluates an estimator against the benchmark estimator $\mathcal{E}_B$ (usually the estimator obtained from the data available to the grand coalition). We can use this to define a reward value for each agent $i \in N$, denoted as $r_i$. The goal is to reward each agent's data contribution with an estimator $\mathcal{E}_i$, such that the quality of $\mathcal{E}_i$ corresponds to the valuation of $i$'s data contribution $D_i$, which will be measured using the function $v$.

\section{Data Valuation Scheme}
We assume coalitions are interested in both the inferred CPDAG $\mathcal{C}$, as well as the estimated distributions of causal effects $\hat{\theta}^{ij}$ for all pairs of variables $(V_i, V_j) \in \mathbf{V} \times \mathbf{V}$. In practice, the ground truth CPDAG and causal effects are unknown quantities and must be inferred from the observed data. The best available estimate for these are the CPDAG and estimated distributions computed using all data available to the grand coalition $N$, due to the asymptotic consistency of the estimators \cite{kalisch2007estimating}. Thus, we will usually treat the grand coalition estimates of the CPDAG $\mathcal{C}_N$ as the surrogate for the CPDAG of the true underlying DAG, similarly, for all $V_i, V_j \in \mathbf{V} \times \mathbf{V}$, we use $\hat{\theta}^{ij}_N$ (the estimated distributions using the grand coalitions' data) as the surrogate for the ground truth population ACE $\theta^{ij}$.

Different CPDAGs often yield vastly different effects. Suppose we have two CPDAGs, $\mathcal{C}_1$ and $\mathcal{C}_2$, over the same set of variables $\mathbf{V}$, and we are interested in the possible causal effects of $V_i \in \mathbf{V}$ on $V_j \in \mathbf{V} \setminus \{V_i\}$. To infer these, one has to determine the sets $\mathbf{Pa}_{V_i}(\mathcal{C}_1)$ and $\mathbf{Pa}_{V_i}(\mathcal{C}_2)$ and has to use these to calculate the possible causal effects (see Section \ref{prelim}).

Therefore, to evaluate the data of a coalition $C \subseteq N$, we propose to first take into account the difference between the CPDAG estimated by that coalition, $\mathcal{C}_C = (V, E_C)$ and the grand coalition estimates of the CPDAG $\mathcal{C}_N = (V, E_N)$. To do this, we introduce a measure extending the measure of \emph{structural intervention distance (SID)}  \cite{peters2014structural}. We will call this the \emph{distribution-SID (dSID)}. It is the number of node-pairs $V_i, V_j$, for which the distribution of causal effects of $V_i$ on $V_j$ given $\mathcal{C}_1$ is different from the distribution given $\mathcal{C}_2$. For a discussion of the SID and other possible distance measures from the literature see Section \ref{rw} on related work. We show how the dSID can be computed in appendix \ref{appdSID}.

\begin{definition}[Distribution Structural Intervention Distance (dSID)]
    Let $\mathbb{G}$ be the space of CPDAGs over $p$ variables. We then define
    \begin{alignat*}{3}
        & dSID:\text{ }  && \mathbb{G} \times \mathbb{G}    && \rightarrow \mathbb{N}\\
        &               && (\mathcal{C}_1, \mathcal{C}_2)    && \mapsto |\{(V_i, V_j) \in \mathbf{V} \times \mathbf{V}, V_i \neq V_j \mid \text{the distribution of causal}\\
        &               &&                                   && \text{effects of $V_i$ on $V_j$ is falsely identified in $\mathcal{C}_1$ with respect to $\mathcal{C}_2$}\}|
    \end{alignat*}
\end{definition}
 This leads to the first part of our data valuation measure. Let $\mathcal{E}$ be an estimator, consisting of a CPDAG $\mathcal{C}_\mathcal{E} = (\mathbf{V}, \mathbf{E})$ and the estimated distributions $\hat{\theta}^{ij}_{\mathcal{E}}$ for all pairs of variables $V_i, V_j \in \mathbf{V} \times \mathbf{V}$. Further, let $\mathcal{B}$ be a benchmark estimator (usually the estimator obtained from dataset $D_N$ of the grand coalition), also consisting of a CPDAG $\mathcal{C}_\mathcal{B} = (\mathbf{V}, \mathbf{E}')$ and distributions $\hat{\theta}^{ij}_{\mathcal{B}}$. Let
\begin{align}
    v_{dSID}(\mathcal{E}, \mathcal{B}) := - \frac{dSID(\mathcal{C}_\mathcal{E}, \mathcal{C}_\mathcal{B})}{|\mathbf{V}| \cdot |\mathbf{V}-1|}.
\end{align} 
The $v_{dSID}$ metric can therefore be seen as the percentage amount of variable pairs, for which the intervention distribution of causal effects differ between $\mathcal{C}_\mathcal{E}$ and $\mathcal{C}_\mathcal{B}$. Its maximum value is 0, if all intervention distributions in the two CPDAGs are identical, and its minimum is $-1$ if all are different.

We will refine the $v_{dSID}$, because even for pairs of variables for which the distribution of causal effects is correctly identified in $\mathcal{C}_\mathcal{E}$ with respect to (wrt) $\mathcal{C}_\mathcal{B}$, the estimated effect distributions $\theta^{ij}_\mathcal{E}$ and $\theta^{ij}_\mathcal{B}$ can differ since they are computed by regression using different datasets. Thus, we adopt the method proposed by Qiao and colleagues \cite{pmlr-v202-qiao23a}, using the negative reverse Kullback-Leibler (KL) divergence. It quantifies, how much one probability distribution differs from another probability distribution. 

 For each pair of variables $V_i, V_j$ where $\mathcal{C}_\mathcal{E}$ and $\mathcal{C}_\mathcal{B}$ identify the same causal effects, the IDA algorithm will yield relatively similar estimated causal effects. In contrast to the approach of Qiao and colleagues \cite{pmlr-v202-qiao23a}, however, for a given variable-pair $(V_i, V_j) \in \mathbf{V} \times \mathbf{V}$, we are not dealing with one causal effect described by a Gaussian distribution, but a distribution over several effects. The mixture model given by the sum of these distributions then describes the effect distribution $\hat{\theta}^{ij}_\mathcal{E}$ (see Equation \ref{mixmodel}). This can then be compared to the mixture model of possible causal effects from the benchmark estimator $\mathcal{B}$, $\hat{\theta}^{ij}_\mathcal{B}$. Define $\mathbf{A} = \{(V_i, V_j) \in \mathbf{V} \times \mathbf{V}, V_i \neq V_j |$ \textit{the distribution of causal effects from} $V_i$ to $V_j$ \textit{is correctly identified in} $\mathcal{C}_\mathcal{E}$ \textit{wrt} $\mathcal{C}_\mathcal{B} \}$. We can now define the second part of our data valuation method:
\begin{align}
    v_{-KL}(\mathcal{E}, \mathcal{B}) = \frac{\sum_{(V_i, V_j) \in \mathbf{A}} -KL\left( \hat{\theta}^{ij}_\mathcal{E}||\hat{\theta}^{ij}_\mathcal{B} \right)}{\big|\mathbf{A}\big|}.
\end{align}

Lastly, we propose the combination of $v_{dSHD}$ and $v_{-KL}$ as the final data valuation function:
\begin{align}
    v(\mathcal{E}, \mathcal{B}) = v_{dSHD}(\mathcal{E}, \mathcal{B}) + v_{-KL}(\mathcal{E}, \mathcal{B}).
\end{align}
When $\mathcal{C}_\mathcal{E}$ and $\mathcal{C}_\mathcal{B}$ identify the same distribution of causal effects for a variable pair $(V_i, V_j) \in \mathbf{V} \times \mathbf{V}$, the underlying data sets must be similar concerning the effects between these variables. Thus the negative KL-divergence between $\hat{\theta}^{ij}_\mathcal{E}$ and $\hat{\theta}^{ij}_\mathcal{B}$ is relatively small. Consequently, $v_{-KL}$-values are typically much smaller than $v_{dSHD}$-values. Because differences in CPDAGs can lead to vastly different causal effects, this is more significant than an effect, that is identically identified by both CPDAGs, but then slightly differently estimated from the regressions given coalitions different data sets. Thus, $v_{dSHD}$ makes up a significantly larger part of our date valuation measure than the relatively small $v_{-KL}$. Intuitively, the $v$ measure can be seen as a measure of the difference of the estimated CPDAGs, degraded by the difference of the estimated effects.

\section{Modeling an Individual Agent}
To understand how agents will behave in a multi-agent setting, we 
 first consider the case of a single agent $i \in N$. 
In this case, the agent only has access to the data he produces himself and the estimator arising from this data. We will assume that each agent $i$ has a marginal fixed cost $c_i > 0$ for producing a data point. Thus his cost for producing data $D_i$ is $cost_i(D_i) = c_i|D_i|$.

Using data $D_i$, agent $i$ will be able to produce some estimator $\mathcal{E}_i$. However, there is no benchmark to assess the quality of $\mathcal{E}_i$ without access to more data. Thus, agent $i$ can only evaluate the quality of his estimator based on estimators he has produced before, using a subset of his current data. Let $T=\{1, 2, ...\}$ be a series of timestep. For each $t \in T$, let $D_i^t$ be the data produced so far by $i$ at timestep $t$ with $D_i^{t-1} \subseteq D_i^t$ for all $t > 1$. Let $\mathcal{E}\left(D_i^t\right) = \mathcal{E}_i^t$ be the estimator obtained by agent $i$ at time $t$ (using data $D_i^t$).

For the rest of this paper, we will assume that an agent will compare his current estimator to the estimators he produced before, using the current estimator as the benchmark (the current estimator presents his best knowledge of the real distribution underlying the data). The agent will act according to how much his estimator improved from older versions. If the improvement rate is judged as too small by the agent, he will decide that further improvements are not worth the cost and thus stop producing further data. We will assume, that for as long as an agent produces data, he will produce the same amount of data at every timestep. Let $\Delta_i$ denote this amount of data, thus for all $t > 1$, $|D_i^t \setminus D_i^{t-1}| = \Delta_i$. So while producing data, at every timestep $i$ occurs additional cost $c_i\Delta_i$. 

$v(\mathcal{E}_i^{t'}, \mathcal{E}_i^t)$ assesses the discrepancy between estimators $\mathcal{E}_i^{t'}$ and $\mathcal{E}_i^{t}$, with 0 indicating identical CPDAGs and causal effect distributions, and -1 indicating completely different CPDAGs. Conversely, $-v(\mathcal{E}_i^{t'}, \mathcal{E}_i^t)$ measures the improvement in estimator quality from $\mathcal{E}_i^{t'}$ to $\mathcal{E}_i^t$, with $0$ denoting no improvement and $1$ completely different estimators.

 While the PC algorithm's asymptotic consistency guarantees eventual convergence to the true underlying distribution with sufficient data  \cite{maathuis2009estimating}, noise in early-stage data may not consistently improve estimator quality over time, Thus, $v(\mathcal{E}_i^{t'}, \mathcal{E}_i^t) > v(\mathcal{E}_i^{t'-1}, \mathcal{E}_i^t)$ is not guaranteed for every $t, t' \in T$ with $t>t'>1$. For this reason, it makes sense for an agent to not only take into account the improvement from the last estimator to the current one, $v(\mathcal{E}_i^{t-1}, \mathcal{E}_i^t)$ but also the average improvement rate from older estimators towards the current one: $\frac{-v(\mathcal{E}_i^{t'}, \mathcal{E}_i^t)}{t-t'}$. With this, we can define the \textit{improvement rate} as follows:
\begin{equation}\label{eq:improvement}
    im_i(t) := \frac{1}{t-1} \sum_{t' \in \{1, ..., t-1\}}\frac{-v(\mathcal{E}_i^{t'}, \mathcal{E}_i^t) + v\left(\mathcal{E}_i^{t}, \mathcal{E}_i^{t}\right)}{t-t'},
\end{equation}
This can be seen as the average rate of improvement in model quality towards $\mathcal{E}_i^t$ per timestep, with $im_i(t) \in [0, 1]$ for all $t$. We assume that $i$ will gain some utility $u_i$ from the estimator improvement. We will assume that the utility increases linearly with increasing estimator quality. From this utility, we subtract the cost of producing $\Delta_i$ data points and thus get agent $i$'s utility at time $t$ by:
\begin{align}
    u_i(t) := im_i(t) - c_i\Delta_i
\end{align}
Finally, we will assume that agent $i$ will continue to produce data as long as his utility is not shrinking and stop producing once his utility is decreasing. That is, at the first $t \in T$, for which $im_i(t) - c_i\Delta_i < 0$, $i$ with stop producing data. Then the behavior of a single agent will be as follows:
\begin{definition}[Optimal Data Production]\label{optdata}
    Let $T=\{1, 2, ...\}$ be an infinite series of timestamps and $i \in N$ an agent. Let $t_i^{s-opt} \in T$ be the unique timestamp, such that for all $t < t_i^{s-opt}$ $im_i(t) - c_i\Delta_i \geq 0$ and $im_i(t_i^{s-opt}) - c_i\Delta_i < 0$. Then the \textup{optimal data production} for $i$ is $D_i^{t_i^{s-opt}}$. For producing this data, $i$ will incur cost $c_i|D_i^{t_i^{s-opt}}| = c_i\Delta_it_{opt}$.
\end{definition}
Since $i$ stops producing further data after $t_i^{s-opt}$, we have $D_i^{t'_i} = D_i^{t_i^{s-opt}}$ for all $t'_i > t_i^{s-opt}$.

\section{Modeling Multiple Agents}
We will now study how agents behave in a collaborative setting as described in Section \ref{CCI}, where multiple agents produce data and the resulting model can be shared. In this framework, a server mediates the interaction between agents. First, the server publishes the mechanism to the agents, then, at each timestep $t$, each agent $i$ contributes some data $D_i^t$ to the server and finally, the server in turn rewards the agents with an estimator $\mathcal{E}_i^t$ of a certain quality. The server also provides agents with the opportunity to check the quality of any estimator they possess, that is, at time $t$ the server computes $v(\mathcal{E}_i^{t'}, \mathcal{E}_N^t)$ for any estimator $\mathcal{E}_i^{t'}$ which $i$ received at some time $t' \leq t$. This way, $i$ can always check the improvement rate of models he received so far, compared to the best available estimator at the time.

For time steps $T=\{1, 2, ...\}$ let $\mathbf{D}^t = \{D_1^t, ..., D_n^t\}$ be the data provided by agents $1$ to $n$ at time t. A \emph{mechanism} 
%facilitating the agents' collaboration 
is formalized as a function 
$ \mathcal{M}(\mathbf{D}):D_1 \times ... \times D_n \rightarrow \mathcal{E}_1 \times ... \times \mathcal{E}_n$, 
which maps the agents' contributions to estimators they receive from the mechanism. At $t$, so far each agent $i$ generated and transmitted data $D_i^t$ to the server. For this, he is rewarded with an estimator $\mathcal{E}_i^t$ of quality $v(\mathcal{E}_i^t, \mathcal{E}_N^t)$. From now on, $\mathcal{E}_i^t$ will denote the estimator $i$ receives at $t$ from the server, while $\tilde{\mathcal{E}}_i^t$ will denote the estimator obtained only from $i$'s data $D_i^t$. The first requirement for such a mechanism is, for it to be feasible: A mechanism which returns an estimator $\mathcal{E}_i^t$ to agent $i$ is said to be \textup{feasible} if for any $i \in N$ and any $D_i^t$, it satisfies $v(\mathcal{E}_i^t, \mathcal{E}_N^t) \leq v(\mathcal{E}_N^t, \mathcal{E}_N^t)$.
Feasibility is necessary to make sure that the mechanism does not return an estimator that is better than the estimator obtained by pooling all available data, $\mathcal{E}_i^t$. But since $\mathcal{E}_i^t$ is always our benchmark to which every other estimator is compared, by definition for any estimator $v(\mathcal{E}', \mathcal{E}_N^t) \leq 1$, while $v(\mathcal{E}_N^t, \mathcal{E}_N^t) = 1$, thus feasibility always holds.

The second requirement for our mechanism is individual rationality (IR), which is defined as follows: Given data contributions $\mathbf{D}^t$ by the agents, the mechanism provides an estimator $\mathcal{E}_i^t$ to agent $i$. Such a mechanism $\mathcal{M}$ is said to satisfy IR if for any agent $i \in N$ and any contribution $\mathbf{D}^t$, it holds that $v(\mathcal{E}_i^t, \mathcal{E}_N^t) \geq v(\tilde{\mathcal{E}}_i^t, \mathcal{E}_N^t)$. 
Thus individual rationality guarantees that no agent will obtain a worse estimator than they would have obtained if they were working independently. IR is necessary to guarantee that all rational agents will participate in the mechanism.

\subsection{Standard Collaborative Setting}
We first examine the most basic setting for collaborative causal inference. Here, each agent immediately has access to the model of the grand coalition, obtained from all data available to the mechanism at time $t$, $\mathcal{E}_N^t$. Thus at all $t \in T$ and for all $i \in N$: $[\mathcal{M}(\mathbf{D}^t)]_i = \mathcal{E}^t_N$. 
This mechanism is feasible and also satisfies individual rationality  since $v(\mathcal{E}_N^t, \mathcal{E}_N^t) \geq v(\tilde{\mathcal{E}}_i^t, \mathcal{E}_N^t)$. We will assume that agents will behave as they would if they were to act on their own. That is, at each time step an agent $i \in N$ will produce and contribute some data to the mechanism and then evaluate, how much the estimator of the grand coalition has improved through their data contribution. If this improvement makes up for the cost $i$ incurred, he will continue to produce data, until the improvement does not outweigh their cost anymore. When first entering the mechanism, $i$ will contribute initial data $D^1_i$, since he does not know yet, how much this data will improve the grand coalition's estimator. After that, at each timestep $t > 1$, $i$ will compute the improvement rate $im_i(t)$ according to Equation \ref{eq:improvement}, but now comparing estimators of the grand coalitions at times $t' < t$ to the current estimator $\mathcal{E}_N^t$.

Again, $i$'s utility is calculated as $u_i(t):= \sum_{t' \in \{1, ..., t\}} \left(im_i(t) - c_i\Delta_i\right)$, and agent $i$ will continue to produce data as long as this is not decreasing and stop producing once it decreases. Thus the behavior of an agent $i$ in the standard collaborative setting will be as described in definition \ref{optdata}. However, we will denote the optimal time for agent $i$ to stop producing further data in the collaborative setting by $t_i^{opt}$, as opposed to $t_i^{s-opt}$ in the single agent setting.

\subsection{Data Maximizing Mechanism}
We will now present a mechanism that maximizes the data provided by all agents $i \in N$. A mechanism $\mathcal{M}$ is \textit{data-maximizing} given costs $\mathbf{c} = \{c_1, ..., c_n\}$ if it maximizes the data collected at equilibrium. Since at each timestep $i$ produces a fixed amount of data $\Delta_i$, the total amount of data provided by agent $i$ is $t^{opt}_i\Delta_i$
\begin{definition}[Data Maximization]\label{defDM}
    A mechanism $\hat{\mathcal{M}}$ is data-maximizing if 
        $\hat{\mathcal{M}} \in \argmax_{\mathcal{M}} \sum_{i \in N} \left[t^{opt}_i \Delta_i\right]$, i.e., if it maximizes the amount of data collected, subject to $\hat{\mathcal{M}}$ being feasible and satisfying IR.
\end{definition}

To maximize contributions, we need to make sure agents produce data for as long as possible. To achieve this, we will reward them at every time step with an estimator which is just good enough for them not to achieve a decrease in utility. We will do this for as long as we can provide such an estimator, given access to all the data contributed by the grand coalition. Following the design described by Karimireddy and colleagues \cite{karimireddy2022mechanisms}, we propose the following mechanism. Now
\begin{align}
    im_i(t) = \frac{1}{t - 1} \sum_{t' \in \{1, ..., t-1\}} \frac{-v\left(\mathcal{E}_i^{t'}, \mathcal{E}_N^t\right) + v\left(\mathcal{E}_i^{t}, \mathcal{E}_N^{t}\right)}{t-t'},
\end{align}
as in equation \ref{eq:improvement}, $im_i(t)$ measures the improvement towards the quality of the last model $i$ has received, $\mathcal{E}_i^t$, however, this does not have to be the best possible model $\mathcal{E}_N^t$. Now we can formulate our data-maximizing mechanism:

\begin{align}\label{datamax}
    [\mathcal{M}(\mathbf{D}^t)]_i =
    \begin{cases}
        \tilde{\mathcal{E}^t_i}  & \text{if } t < t^{s-opt}_i\\
        \mathcal{E}^t_i, \text{such that } im_i(t) =c_i\Delta_i + \epsilon   & \text{else, if such $\mathcal{E}^t_i$ exists}\\
        \mathcal{E}_N & \text{else}.
    \end{cases}
\end{align}
Even without external incentivization, $i$ will continue to produce data until $t^{s-opt}_i$. From here on, $i$ will receive a model $\mathcal{E}^t_i$ such that the improvement rate just makes up for his incurred cost, so that his utility does not decrease. At some point, $i$ will receive the best possible model $\mathcal{E}^t_N$ or something very close to this, so it is no longer possible to keep the average improvement rate over $c_i\Delta_i$. Now, $i$ will receive the best possible model and no further incentivization is possible. This leads to the result expressed in the following theorem.  
\begin{theorem}[Data maximization with known costs]\label{theo1}
    The mechanism $\mathcal{M}$ defined by equation \ref{datamax} is data-maximizing for $\epsilon \rightarrow 0^+$. A rational agent $i$ will contribute $\Delta_it_i^{opt}$ data points where $\Delta_it_i^{opt} \geq \Delta_it_i^{s-opt}$, yielding a total of $\sum_{j \in N} \Delta_jt_j^{opt}$ data points.
\end{theorem}
A proof for Theorem \ref{theo1} is presented in Appendix \ref{proofs}.

\subsection{Achieving Fairness}
While the mechanism described above maximizes the amount of data extracted from each agent, it does not satisfy fairness. In parallel to work done by Qiao and colleagues \cite{pmlr-v202-qiao23a}, we can use our data valuation function $v(\mathcal{E}, \mathcal{B})$, to design a reward scheme $r_i^t$ for every agent $i \in N$ at each time $t \in T$, which fulfills the following fairness conditions defined in the work of Qiao and colleagues \cite{pmlr-v202-qiao23a}. At each $t \in T$, let $v^t_{\emptyset} := min _{i \in N}v(\mathcal{E}^t_i, \mathcal{E}^t_N)$ for the empty coalition $C = \emptyset$. Then: 
        \begin{itemize}
            \item \textbf{(F1) Uselessness:} Agent $i$ should receive a valueless reward if his data does not improve the estimator of any other coalitions.
           \item \textbf{(F2) Symmetry:} If two agents yield the same improvement for all other coalitions, then they should receive equally valuable estimates as rewards.
            \item \textbf{(F3) Strict Desirability:} If data from agent $i$ strictly improves the estimate for at least one coalition more than that of agent $j$, but the reverse is not true, then agent $i$ should receive a more valuable reward than agent $j$.
            \item \textbf{(F4) Monotonicity:} For an agent $i$, if his dataset $D_i$ strictly improves the estimate for at least one coalition more compared to that of another dataset $D_{i'}$, but the reverse is not true, then sharing $D_i$ should give agent $i$ more valuable reward than sharing $D_{i'}$.
        \end{itemize} 
At each $t \in T$, let $v^t_{\emptyset} := min _{i \in N}v(\mathcal{E}^t_i, \mathcal{E}^t_N)$ for the empty coalition $C = \emptyset$. The marginal contribution $m^t_i(T)$ of the data of agent $i$ to a coalition $T \subseteq N$ at time $t$ is then defined as $m^t_i(T) := v(\mathcal{E}^t_{T \cup \{i\}}, \mathcal{E}^t_N) - v(\mathcal{E}^t_T, \mathcal{E}^t_N)$, 
with $\mathcal{E}^t_N$ being the estimator obtained from all data available to the grand coalition $N$. As proposed by Qiao and colleagues, the Shapley Value (SV) can then be calculated as 
    $\phi^t_i := \frac{1}{|N|!} \sum_{T \subseteq N \setminus \{i\}} |T|! \cdot \left( |N| - |T| - 1 \right)! \cdot m^t_i(T)$. 
Using this, the modified $\rho$-Shapley fair reward value is defined by
\begin{align}
    r^t_i := max \left\{ v\left(\mathcal{E}^t_i, \mathcal{E}^t_N\right) - v^t_{\emptyset}, -v^t_{\emptyset} \cdot \left(\frac{\phi^t_i}{\phi^{t*}}\right)^{\rho} \right\},
\end{align}
with scaling factor $\rho \in (0, 1]$ and $\phi^{t*} := max_{i \in N} \phi^t_i$. Using this keeps rewards lower than when using the standard SV while still ensuring the fairness conditions \cite{pmlr-v202-qiao23a}. Thus, agents can be incentivized to produce data over a longer period than by using the standard SV. For $\rho \rightarrow 0$, $r^t_i \rightarrow 0$, thus all agents receive the same reward. The larger $\rho$, the larger the differences between agent's rewards, and thus the more fairness is achieved. A detailed discussion on the effect of the $\rho$ parameter has been done by Qiao and colleagues \cite{pmlr-v202-qiao23a}.

We can use this reward scheme, to ensure that agents with more valuable data receive a better estimator than other agents. For this, when computing the modified $\rho$-Shapley fair reward value $r^t_i$, we need to set $\rho \leq \min_{i \in N} log(1 - v\left(\mathcal{E}^t_i, \mathcal{E}^t_N\right) / v^t_{\emptyset})/ log(\phi^t_i / \phi^{t*})$ \cite{pmlr-v202-qiao23a}. Now we can define our fair mechanism as:
\begin{align}
    [\mathcal{M}_{fair}(\mathbf{D}^t)]_i =
    \begin{cases}
        \tilde{\mathcal{E}^t_i}  & \text{if } t < t^{s-opt}_i\\
        \mathcal{E}^t_i, \text{such that } im_i(t) =c_i\Delta_i + r^t_i   & \text{else, if such $\mathcal{E}^t_i$ exists}\\
        \mathcal{E}_N & \text{else}.
    \end{cases}
\end{align}

\begin{proposition}
    $\mathcal{M}_{fair}$ ensures fairness since agents with more valuable data receive a better estimator \cite{pmlr-v202-qiao23a}. The agent with the least valuable data receives an estimator which just ensures that his average improvement rate balances his costs, agents with more valuable data receive estimators with larger improvement rates. This, in turn, means that $\mathcal{M}_{fair}$ is not data maximizing anymore, because agents with data of high value receive better estimators faster than they would under $\mathcal{M}$. 
\end{proposition}
This proposition follows directly from the construction of the reward scheme and the fact that $\mathcal{M}$, as defined in equation \ref{datamax}, is data maximizing with $\mathcal{M} \neq \mathcal{M}_{fair}$.

\section{Related Work}\label{rw}
The data valuation scheme presented in this paper is inspired by the scheme presented by Qiao and colleagues \cite{pmlr-v202-qiao23a}, in particular w.r.t.\ the use of negative reverse KL divergence. 
However, they only performed one regression to estimate one effect, they did not estimate a CPDAG and did not estimate the causal effects between all pairs of variables.

For this reason, a measure for distances between CPDAGs was needed. One possible measure here would be the structural hamming distance (SHD). However, as described by Peters and Bühlmann \cite{peters2014structural}, the SHD does not necessarily describe well, how different two graphs are regarding the causal effects they decode. They therefore introduce the structural intervention distance (SID), which, given two DAGs $\mathcal{H}$ and $\mathcal{G}$ over the same vertices $\mathbf{V}$, describes how many intervention distributions between variable pairs are falsely estimated by $\mathcal{H}$ wrt $\mathcal{G}$. The SID can be applied to CPDAGs as well, here however it returns a range of the SIDs between all possible DAG extensions of the CPDAGs, whereas here a definite measure was needed.

A data maximizing mechanism for federated learning was first presented by Karimireddy and colleagues \cite{karimireddy2022mechanisms}. They assumed, however, that each data point provided by any agent had the same value and that the quality of a model could be directly computed. In our setting of collaborative causal inference, these assumptions do not hold anymore, therefore the data valuation scheme was introduced, both to measure the value of an agent's data as well as to measure the quality of a given estimator.

\section{Conclusion and Future Work}
This paper presents novel mechanisms for CCI aimed at motivating the engagement of self-interested agents through the provision of superior causal effect estimators compared to those agents that could develop independently. Our initial mechanism incentivizes all participating agents to maximize data contribution, thereby facilitating the construction of a comprehensive model based on extensive data and expected to yield high accuracy. Building upon this, our second mechanism assigns better estimators to agents supplying higher-quality data, ensuring a fair reward allocation given individual contributions.

A limitation of our mechanism lies in the assumption of known data production costs $c_i$ for agents, prompting a future research direction towards maximizing data despite unknown $c_i$'s. Furthermore, the presumption of honesty and non-malicious behavior among all agents may not hold in practical scenarios, with some parties potentially exploiting the transparent framework for personal gain or harm to others. Consequently, a more robust mechanism resilient to such behavior would be preferable.

\subsection*{Acknowledgements}
Björn Filter acknowledges funding for project FPOplus  by BMDV/TÜV Rheinland Consulting GmbH.

\bibliographystyle{splncs03_unsrt.bst}
\bibliography{refs}

%\newpage
\appendix
\section{Distribution Structural Intervention Distance}\label{appdSID}
To see how the dSID can be computed, we will first need some further preliminaries, consisting of some more graph-theoretical definitions and the concept of covariate adjustment.

\subsection{Further Preliminaries}
A \textit{partially directed acyclic graph (PDAG)} is a partially directed graph without directed cycles. A PDAG $\mathcal{G}$ is a \textit{maximally oriented PDAG (MPDAG)} if and only if the edge orientations in $\mathcal{G}$ are complete under the Meek orientation rules \cite{MeekRules}.

To identify the total causal effect of a variable $V_i \in \mathbf{V}$ on a variable $V_j \in \mathbf{V}$, one has to find an expression for $P(y \mid do(x))$. Using only the preintervention probabilities of the observed variables $\mathbf{V}$, this can be done through \textit{covariate adjustment}. Here, a so-called \textit{adjustment set} is used, to estimate the causal effect \cite{Pearl09}. A graph-theoretical criterion can be used to determine, whether a set $\mathbf{Z} \subseteq \mathbf{V} \mathbin{\backslash} \lbrace V_i, V_j \rbrace$ is an adjustment set wrt two nodes $V_i, V_j \in \mathbf{V}$. This is the so-called adjustment criterion:

\begin{definition}[Adjustment Criterion for DAGs \cite{Shpitser2010}]\label{ACdef}
Let $\mathcal{D}  = \mathbf{(V, E)} $ be a DAG, $V_i, V_j \in \mathbf{V}$ two nodes and $\mathbf{Z} \subseteq \mathbf{V} \mathbin{\backslash} \lbrace V_i, V_j \rbrace$. The set $\mathbf{Z}$ satisfies the adjustment criterion relative to $(V_i, V_j)$ in $\mathcal{D}$ if 
\begin{enumerate}[label=(\alph*)]
\item
no element in $\mathbf{Z}$ is a descendant in $\mathcal{D}$ of any $W \in \mathbf{V} \mathbin{\backslash} V_i$ which lies on a causal path from $V_i$ to $V_j$ and
\item
all  non-causal paths in $\mathcal{D}$ from $V_i$ to $V_j$ are d-separated by $\mathbf{Z}$.
\end{enumerate}
\end{definition}

The adjustment criterion can then be used to find adjustment sets:

\begin{theorem}[\cite{Shpitser2010}]
Let $\mathcal{D}  = \mathbf{(V, E)} $ be a DAG, $V_i, V_j \in \mathbf{V}$ two nodes and $\mathbf{Z} \subseteq \mathbf{V} \mathbin{\backslash} \lbrace V_i, V_j \rbrace$. $\mathbf{Z}$ is an adjustment set wrt $(V_i, V_j)$ in $\mathcal{D}$, if and only if $\mathbf{Z}$ satisfies the adjustment criterion (definition \ref{ACdef}) relative to $(V_i, V_j)$ in $\mathcal{D}$.
\end{theorem}

For a DAG $\mathcal{D} = (\mathbf{V, E})$ and two nodes $V_i, V_j \in \mathbf{V}$ let $AS_{(V_i, V_j)}(\mathcal{D}) = \lbrace \mathbf{Z} \mid \mathbf{Z} \subseteq \mathbf{V} \mathbin{\backslash} \lbrace V_i, V_j \rbrace \textit{ and } \mathbf{Z}$ \textit{fulfills the adjustment criterion for} $\mathcal{D} \textit{ wrt } (V_i, V_j) \rbrace$ be the set of all possible adjustment sets wrt $(V_i, V_j)$ in $\mathcal{D}$. It is easy to see that $Pa_{V_i}(\mathcal{D}) \in AS_{(V_i, V_j)}(\mathcal{D})$, because $Pa_{V_i}(\mathcal{D})$ always fulfills the adjustment criterion wrt $(V_i, V_j)$.

\subsection{An Algorithm for Computing the dSID}
Recall, that the dSID is defined as follows: Let $\mathcal{C}_1$ and $\mathcal{C}_2$ be two CPDAGs over the same set of variables $\mathbf{V}$.
\begin{alignat*}{3}
 & dSID:\text{ }  && \mathcal{G} \times \mathcal{G}    && \rightarrow \mathbb{N}\\
        &               && (\mathcal{C}_1, \mathcal{C}_2)    && \mapsto |\{(V_i, V_j) \in \mathbf{V} \times \mathbf{V}, V_i \neq V_j \mid \text{the distribution of causal}\\
        &               &&                                   && \text{effects of $V_i$ on $V_j$ is falsely identified in $\mathcal{C}_1$ with respect to $\mathcal{C}_2$}\}|
\end{alignat*}

with the distribution of causal effects of $V_i$ on $V_j$ given CPDAG $\mathcal{C}$ defined as:
\begin{align}
    P_{\mathcal{C}}(v_j \mid do(V_i = v_i)) = \sum_{\mathcal{D}\in CE(\mathcal{C})}\frac{1}{|CE(\mathcal{C})|} P_{\mathcal{D}}(v_j \mid do(V_i = v_i))
\end{align}

To compute the dSID, for all $(V_i, V_j) \in \mathbf{V} \times \mathbf{V}$, we first need to identify all possible effects of $V_i$ on $V_j$ in all DAG-extensions of both $\mathcal{C}_1$ and $\mathcal{C}_2$. We then need to check, whether all effects found given $\mathcal{C}_1$ can be found given $\mathcal{C}_2$ as well, and vice versa. The dSID Algorithm is depicted in figure \ref{dSID}. It will be described in detail in the following.

As mentioned before, to identify the possible effects of $V_i$ on $V_j$ given a CPDAG $\mathcal{C}$, we can compute $\mathbf{Pa}_{V_i}(\mathcal{C})$ and use these as adjustment sets to compute the effects. However, suppose $\mathcal{D}, \mathcal{D'} \in CE(\mathcal{C})$ with $Pa_{V_i}(\mathcal{D}) \neq Pa_{V_i}(\mathcal{D'})$. It is still possible, that $Pa_{V_i}(\mathcal{D'})$ can be an adjustment set wrt $(V_i, V_j)$ in $\mathcal{D}$ (this is the case, when $Pa_{V_i}(\mathcal{D'})$ fulfills the adjustment criterion wrt $(V_i, V_j)$ in $\mathcal{D}$). Now, the effect of $V_i$ on $V_j$ would be identical given both $\mathcal{D}$ and $\mathcal{D'}$. We will therefore start by identifying the subclasses of DAGs, such that the effect of $V_i$ on $V_j$ in each subclass is identical. For this, the IDGraphs Algorithm \cite{guo2021minimal} can be used (line 3 in figure \ref{dSID}), which runs in $O(2^{m(\mathcal{C})})poly(|\mathbf{V}|)$, where $m(\mathcal{C})$ is the number of undirected edges incident to $V_i$ on a proper possibly causal path from $V_i$ to $V_j$. While $m(\mathcal{C})$ could be potentially as big as $p$, in practice it tends to be small. $poly(|\mathbf{V}|)$ describes the time used to complete the orientation rules of Meek \cite{MeekRules}, which are also used in the PC algorithm. In practice, IDGraphs takes about twice the time of IDA. It returns the minimal set of MPDAGs $\mathbf{M}$, such that $CE(\mathcal{C}) = \cup_{\mathcal{M}\in \mathbf{M}}CE(\mathcal{M})$ and for all $\mathcal{M} \in \mathbf{M}$, for all pairs $(\mathcal{D}, \mathcal{D'}) \in \mathcal{M}$, $Pa_{V_i}(\mathcal{D})$ is an adjustment set wrt $(V_i, V_j)$ in $\mathcal{D'}$.

\begin{figure}[H]
    \begin{algorithm}[H]
        \SetKwInput{KwInput}{Input}
        \SetKwInput{KwOutput}{Output} 
        \KwInput{CPDAGs $\mathcal{C}_1 = (\mathbf{V, E_1})$, $\mathcal{C}_2 = (\mathbf{V, E_2})$}
        \KwOutput{Number of pairs $(V_i, V_j) \in \mathbf{V} \times \mathbf{V}$, for which the distribution of causal effects of $V_i$ on $V_j$ is falsely identified in $\mathcal{C}_1$ with respect to $\mathcal{C}_2$}
        dSID $\leftarrow$ 0\;
        \ForEach{$(V_i, V_j) \in \mathbf{V} \times \mathbf{V}$}{%
            $\mathbf{M}_1 =$ \textbf{IDGraphs}($\mathcal{C}_1, V_i, V_j$), $\mathbf{M}_2 =$ \textbf{IDGraphs}($\mathcal{C}_2, V_i, V_j$)\;
            \ForEach{$\mathcal{M} \in \mathbf{M}_1$}{%
                found $\leftarrow$ false\;
                $\mathcal{D} \leftarrow$ any DAG $\mathcal{D} \in CE(\mathcal{M})$\;
                \ForEach{$\mathcal{M}' \in \mathbf{M}_2$}{%
                    $\mathcal{D}' \leftarrow$ any DAG $\mathcal{D}' \in CE(\mathcal{M}')$\;
                    \If{$Pa_{V_i}(\mathcal{D})$ is an adjustment set wrt $(V_i, V_j)$ for $\mathcal{D}'$}{%
                        \If{$\frac{|CE(\mathcal{M})|}{|CE(\mathcal{C}_1)|} = \frac{|CE(\mathcal{M}')|}{|CE(\mathcal{C}_2)|}$}{%
                            found $\leftarrow$ true\;
                            \textbf{break}\;
                        }
                    }
                }
                \If{$\neg$found}{%
                    dSID $+=1$\;
                    \textbf{break}\;
                }
            }
        }
    \KwRet{dSID}
    \caption{dSID}
    \end{algorithm}
    \caption{The dSID algorithm}
    \label{dSID}
\end{figure}

Let $\mathbf{M}_i$ be the results of running IDGraphs on $\mathcal{C}_i$ for $i \in \{1, 2\}$. We can use these sets to compute the possible effects and their probabilities: For each $\mathcal{M} \in \mathbf{M}_1$, sample any $\mathcal{D} \in \mathcal{M}$ (line 6). $Pa_{V_i}(\mathcal{D})$ can be used to compute the effect of $V_i$ on $V_j$. The effects probability is given by $\frac{CE(\mathcal{M})}{CE(\mathcal{C}_1)}$. Now we need to check whether this effect occurs also given $\mathcal{C}_2$, and with the same probability. To do this, we will need the following Theorem:

\begin{theorem}\label{theo2a}
    Let $\mathcal{D}_1, \mathcal{D}_2$ be two Markov equivalent DAGs with nodes $\mathbf{V}$ and let $V_i, V_j \in \mathbf{V}$ be two nodes. Then $AS_{(V_i, V_j)}(\mathcal{D}_1) \cap AS_{(V_i, V_j)}(\mathcal{D}_2) \neq \emptyset$, if and only if $AS_{(V_i, V_j)}(\mathcal{D}_1) = AS_{(V_i, V_j)}(\mathcal{D}_2)$.
    \end{theorem}
    
    \begin{proof}
    WLOG, suppose there are two sets $\mathbf{Z}, \mathbf{Z}' \subseteq \mathbf{V}\mathbin{\backslash} \lbrace V_i, V_j \rbrace$ with $\mathbf{Z} \in AS_{(V_i, V_j)}(\mathcal{D}_1) \cap AS_{(V_i, V_j)}(\mathcal{D}_2)$ and $\mathbf{Z}' \in AS_{(V_i, V_j)}(\mathcal{D}_1)$.We will show, that $\mathbf{Z}' \in AS_{(V_i, V_j)}(\mathcal{D}_2)$. According to \cite{Shpitser2010}, 
    \begin{align}
    \mathbf{Z}' \in AS_{(V_i, V_j)}(\mathcal{D}_2) & \Leftrightarrow P_{M_2}(v_j \mid do(v_i))\\
    & = \sum_{\mathbf{Z}'}P_{M_2}(v_j \mid v_i, z') P_{M_2}(z')
    \end{align}
    in all distributions $P_{M_2}$ entailed by a Markovian SCM $M_2$ which induces $\mathcal{D}_2$. 
    Since $\mathcal{D}_1$ and $\mathcal{D}_2$ are Markov equivalent, every distribution $P_{M_2}$ is also entailed by some Markovian SCM $M_1$ inducing $\mathcal{D}_1$ and vice versa (see Proposition 7.1 in \cite{PetersJanzingSchoelkopf17}). This gives
    \begin{align}
    \centering
    P_{M_2}(v_j \mid do(v_i)) & = \sum_{\mathbf{Z}}P_{M_2}(v_j \mid v_i, \mathbf{z}) P_{M_2}(\mathbf{z}) && \text{($\mathbf{Z} \in AS_{(V_i, V_j)}(\mathcal{D}_2)$)}&\\[1.25ex]
    & = \sum_{\mathbf{Z}}P_{M_1}(v_j \mid v_i, \mathbf{z}) P_{M_1}(\mathbf{z}) && \text{($P_{M_2}$ entailed by $M_1$)}&\\[1.25ex]
    & = P_{M_1}(v_j \mid do(v_i)) && \text{($\mathbf{Z} \in AS_{(V_i, V_j)}(\mathcal{D}_1)$)}&\\[1.25ex]
    & = \sum_{\mathbf{Z}'}P_{M_1}(v_j \mid v_i, \mathbf{z'}) P_{M_1}(\mathbf{z'}) && \text{($\mathbf{Z'} \in AS_{(V_i, V_j)}(\mathcal{D}_1)$)}&\\[1.25ex]
    & = \sum_{\mathbf{Z}'}P_{M_2}(v_j \mid v_i, \mathbf{z'}) P_{M_2}('\mathbf{z'}) && \text{($P_{M_1}$ entailed by $M_2$)}
    \end{align}
    and thus follows $\mathbf{Z}' \in AS_{(V_i, V_j)}(\mathcal{D}_2)$.
\end{proof}

Using this, we can check, whether all effects found given $\mathcal{C}_1$ can also be found given $\mathcal{C}_2$ with the same probability. For each $\mathcal{M} \in \mathbf{M}_1$, we can find the causal effect by sampling a DAG $\mathcal{D} \in \mathcal{M}$ and using $Pa_{V_i}(\mathcal{D})$ as an adjustment set. The IDGraphs algorithm ensured, that $Pa_{V_i}(\mathcal{D})$ is an adjustment set wrt $(V_i, V_j)$ for all $\mathcal{D}_j \in \mathcal{M}$ and therefore, by theorem \ref*{theo2a}, for all $\mathcal{D}_1, \mathcal{D}_2 \in \mathcal{M}$, $AS_{(V_i, V_j)}(\mathcal{D}_1) = AS_{(V_i, V_j)}(\mathcal{D}_2)$.

Now, iterate over all $\mathcal{M}' \in \mathbf{M}_2$ (line 7), sample a DAG $\mathcal{D}' \in \mathcal{M}'$ (line 8) and check, whether $Pa_{V_i}(\mathcal{D})$ fulfills the adjustment criterion wrt $(V_i, V_j)$ in $\mathcal{D}'$ (line 9). This can be efficiently done by the Bayes-Ball algorithm \cite{shachter2013bayes}. If this is the case, by theorem \ref{theo2a}, we know that, $AS_{(V_i, V_j)}(\mathcal{D}) = AS_{(V_i, V_j)}(\mathcal{D}')$. By the properties of IDGraphs and theorem \ref*{theo2a}, for all $\mathcal{D}_1', \mathcal{D}_2' \in \mathcal{M}'$, $AS_{(V_i, V_j)}(\mathcal{D}_1') = AS_{(V_i, V_j)}(\mathcal{D}_2')$. Thus, finally, for all $\mathcal{D} \in \mathcal{M}$ and ${D}' \in \mathcal{M}'$, $AS_{(V_i, V_j)}(\mathcal{D}) = AS_{(V_i, V_j)}(\mathcal{D}')$. In other words, $P_{\mathcal{D}}(v_j \mid do(V_i = v_i))$ will be identical given any DAG $\mathcal{D} \in \mathcal{M} \cup \mathcal{M}'$. If so, check whether $\frac{|CE(\mathcal{M})|}{|CE(\mathcal{C}_1)|} = \frac{|CE(\mathcal{M}')|}{|CE(\mathcal{C}_2)|}$ (line 10), that is not only the effects are identical but also their probability given $\mathcal{C}_1$ and $\mathcal{C}_2$. This can be done using the Clique-Picking algorithm \cite{wienöbst2023efficient}, which computes the number of DAG-extensions of a given MPDAG in polynomial time. If for all $\mathcal{M} \in \mathbf{M}_1$, a corresponding $\mathcal{M}' \in \mathbf{M}_2$ can be found, $P_{\mathcal{C}_1}(v_j \mid do(V_i = v_i)) = P_{\mathcal{C}_2}(v_j \mid do(V_i = v_i))$ must hold.

Applying this procedure for all pairs $(V_i, V_j) \in \mathbf{V} \times \mathbf{V}$ and counting the instances, in which $P_{\mathcal{C}_1}(v_j \mid do(V_i = v_i))$ and $P_{\mathcal{C}_2}(v_j \mid do(V_i = v_i))$ differ (line 14), results in the dSID. In the worst case, the runtime can amount to $p^2O(2^{m(C)})poly(|V|)2^p2^pO(V)ploy(|V|)=O(2^{2p})$. Here, $p^2$ results from going over all pairs of variables, $O(2^{m(C)})poly(|V|)$ is the runtime of IDGraphs, then there can be up to $2^p$ elements in both $\mathbf{M}_1$ and $\mathbf{M_2}$, the Bayes-Ball algorithm takes $O(|V|)$ time. Finally, the Clique-Picking algorithm runs in $ploy(|V|)$. However, normally CPDAGs tend to be sparse and mostly directing, resulting in very small sets $\mathbf{M}_1$ and $\mathbf{M_2}$. In these cases, IDGraphs and Clique-Picking run fast as well and the algorithm is not significantly slower than IDA.

\section{Proof of Theorem \ref{theo1}}\label{proofs}
\begingroup
\def\thetheorem{\ref{theo1}}
\begin{theorem}[Data maximization with known costs]
    The mechanism $\mathcal{M}$ defined by equation \ref{datamax} is data-maximizing for $\epsilon \rightarrow 0^+$. A rational agent $i$ will contribute $\Delta_it_i^{opt}$ data points where $\Delta_it_i^{opt} \geq \Delta_it_i^{s-opt}$, yielding a total of $\sum_{j \in N} \Delta_jt_j^{opt}$ data points.
\end{theorem}
\addtocounter{theorem}{-1}
\endgroup
\begin{proof}
    We will define the best response of agent $i$ to a mechanism $\mathcal{M}$ and data contributions of all other agents over time $\mathbf{D}_{-i} = \{\mathbf{D}_{-i}^1, \mathbf{D}_{-i}^2 ...\}$ as the time, at which $i$ will stop to produce data:
    \begin{align}
        B^{\mathcal{M}}_i(\mathbf{D}_{-i}) := t_i^{opt},
    \end{align}
    such that for all $t < t_i^{opt}$ $im_i(t) \geq c_i\Delta_i$ and $im_i(t_i^{opt}) < c_i\Delta_i$.

    First, we will see that given fixed data contributions $\mathbf{D}_{-i}^t$ from other users for all $t \in T$, for each agent $i$, $i$'s best response to our mechanism $\mathcal{M}$ consists of more data than to any other feasible and IR mechanism $\hat{\mathcal{M}}$. Will will then continue to show, that this implies that the equilibrium contribution of the agent is also data maximizing.
    \begin{lemma}\label{lem1}
        For given data contributions over time $D_{-i}$ and any feasible and IR mechanism $\hat{\mathcal{M}}$, define best responses of agent $i$ $B^{\mathcal{M}}_i(\mathbf{D}_{-i})$ and $B^{\hat{\mathcal{M}}}_i(\mathbf{D}_{-i})$ for our mechanism $\mathcal{M}$ as defined in equation \ref{datamax} and the other mechanism $\hat{\mathcal{M}}$. Then, for any $i \in N$ and any $D_{-i}^t$,
        \begin{align}\label{eqlem1}
            B^{\mathcal{M}}_i(\mathbf{D}_{-i}) \geq B^{\hat{\mathcal{M}}}_i(\mathbf{D}_{-i})
        \end{align}
    \end{lemma}
    For now, we will assume the above lemma and continue with our proof. Since the best responses are defined as the time, at which $i$ stops contributing more data, we will write $B^{\hat{\mathcal{M}}}_i(\mathbf{D}_{-i}) = t_i^{\hat{\mathcal{M}}}$ and $B^{\mathcal{M}}_i(\mathbf{D}_{-i}) = t_i^{\mathcal{M}}$. 
    By definition \ref{defDM}, the total data collected by mechanism $\mathcal{M}$ is calculated as
    \begin{align}
        \sum_{i \in N} t_i^{\mathcal{M}}\Delta_i = \sum_{i \in N} B^{\mathcal{M}}_i(\mathbf{D}_{-i})\Delta_i.
    \end{align}
    Thus, from equation \ref{eqlem1} directly follows:
    \begin{align}
        \sum_{i \in N} B^{\mathcal{M}}_i(\mathbf{D}_{-i})\Delta_i = \sum_{i \in N} t_i^{\mathcal{M}}\Delta_i \geq \sum_{i \in N} B^{\hat{\mathcal{M}}}_i(\mathbf{D}_{-i})\Delta_i = \sum_{i \in N} t_i^{\hat{\mathcal{M}}}\Delta_i
    \end{align}
    for all mechanisms $\hat{\mathcal{M}}$ which are feasible and IR. Thus $\mathcal{M}$ is data maximizing.

    It is easy to see that $\mathbf{t}^{\mathcal{M}} = \{ t_i^{\mathcal{M}} | i \in N\}$ is the unique Nash equilibrium: Each agent $i \in N$, will continue to produce data at least until $t_i^{s-opt}$, since here their behavior is identical to as if they would act on their own. From $t_i^{s-opt}$ on, by definition of $\mathcal{M}$, they do not receive a decrease in utility at every $t < t_i^{\mathcal{M}}$, since $j_i(t) \geq c_i\Delta_i$ for all $t < t_i^{\mathcal{M}}$. Thus $i$ will continue to produce data until $t_i^{\mathcal{M}}$. Due to $j_i(t_i^{\mathcal{M}}) < c_i\Delta_i$, $i$ will stop producing data at this time. Therefore, $\mathbf{t}^{\mathcal{M}}$ is the unique Nash equilibrium.

\end{proof}
\qed

\begin{proof}[of Lemma 1]
    Suppose, Lemma \ref{lem1} does not hold. Then there has to be some agent $i$, such that $B^{\hat{\mathcal{M}}}_i(\mathbf{D}_{-i}) > B^{\mathcal{M}}_i(\mathbf{D}_{-i})$. Thus $t_i^{\hat{\mathcal{M}}} > t_i^{\mathcal{M}}$. Let $\mathcal{E}_i^{\mathcal{M}, t}$ denote the estimator returned to $i$ by mechanism $\mathcal{M}$ at time $t$ and $\mathcal{E}_i^{\hat{\mathcal{M}}, t}$ denote the estimator returned by $\hat{\mathcal{M}}$. Similarly define $im_i^{\mathcal{M}}(t)$ and $im_i^{\hat{\mathcal{M}}}(t)$.

    Since $\hat{\mathcal{M}}$ satisfies IR, for all $t < t_i^{s-opt}$, $\hat{\mathcal{M}}$ has to return an estimator to $i$ which is at least as good as $\tilde{\mathcal{E}}_i^t$ (which is returned by $\mathcal{M}$), so for all $t < t_i^{s-opt}$: $v\left(\mathcal{E}_i^{\hat{\mathcal{M}}, t}, \mathcal{E}_N^t\right) \leq v\left(\mathcal{E}_i^{\mathcal{M}, t}, \mathcal{E}_N^t\right)$.

    For $t \in [t_i^{s-opt} , t_i^{\mathcal{M}}-1]$, $\mathcal{E}_i^{\mathcal{M}, t}$ is defined as the estimator which fulfills $im_i(t) = c_i\Delta_i$, for as long as such an estimator exists. Thus, $\mathcal{E}_i^{\mathcal{M}, t}$ is defined as the worst possible estimator which will still incentivize $i$ to continue providing data. Therefore, also for all $t \in [t_i^{s-opt} , t_i^{\mathcal{M}}-1]$ : $v\left(\mathcal{E}_i^{\hat{\mathcal{M}}, t}, \mathcal{E}_N^t\right) \leq v\left(\mathcal{E}_i^{\mathcal{M}, t}, \mathcal{E}_N^t\right)$. At $t_i^{\mathcal{M}}$, $i$ stopped producing data under $\mathcal{M}$, since $\mathcal{M}$ was not able to return a model which would satisfy $im_i(t) \geq c_i\Delta_i$ and therefore returned $\mathcal{E}_N^{t_i^{\mathcal{M}}}$. $\hat{\mathcal{M}}$ however was able to provide an estimator $\mathcal{E}_i^{t_i^{\hat{\mathcal{M}}}}$ which could satisfy $im_i(t) \geq c_i\Delta_i$, since under $\hat{\mathcal{M}}$, $i$ continues to produce data for at least one more time. But this would imply 
    \begin{align}\label{eqcontradict}
        1 = v\left(\mathcal{E}_N^{t_i^{\mathcal{M}}}, \mathcal{E}_N^{t_i^{\mathcal{M}}}\right) < v\left(\mathcal{E}_i^{\hat{\mathcal{M}}, t_i^{\mathcal{M}}}, \mathcal{E}_N^{t_i^{\mathcal{M}}}\right).
    \end{align}
    Notice, that since $\Delta_i$  is fixed for all $i$ regardless of the mechanism, until $t_i^{\mathcal{M}}$, all agents have provided the same data to both mechanisms, thus, for $t \leq t_i^{\mathcal{M}}$ both mechanisms evaluate estimators using the same benchmark $\mathcal{E}_N^t$. But then, the inequality described by equation \ref{eqcontradict} is impossible, since by definition of the function $v$, $v\left(\mathcal{E}_N^{t_i^{\mathcal{M}}}, \mathcal{E}_N^{t_i^{\mathcal{M}}}\right) = 1$ and $v$ can never take on a value larger then $1$. Thus we have reached a contradiction.
\end{proof}
\qed

\end{document}